%% file: main.tex
\pdfoutput=1
\RequirePackage[T1]{fontenc}
\documentclass[12pt]{article}
\usepackage[skip=1em plus 0.2em minus 0.1em]{parskip}
\usepackage[height=8.85in,width=6.45in]{geometry}

\usepackage[utf8]{inputenc}
\usepackage{amsmath}
\usepackage{amssymb}
\usepackage{mathtools}
\numberwithin{equation}{section}
\usepackage{slashed}
\usepackage{braket}
\usepackage{enumitem}
\usepackage[svgnames]{xcolor}
\usepackage[colorlinks,citecolor=DarkGreen,linkcolor=FireBrick,urlcolor=FireBrick,linktocpage,unicode]{hyperref}
\usepackage{subcaption}

\usepackage{tocloft}

\usepackage{footmisc}
\input{macros.tex}

\definecolor{blue}{named}{black}

\newcommand*{\email}[1]{\footnote{\href{mailto:#1}{\texttt{#1}}}}

\setlist[itemize,enumerate]{
  parsep=\parskip,                                   %
  itemsep=\dimexpr .3em - \parskip\relax plus 2pt,   %
  topsep=\dimexpr 6pt - \parskip\relax plus 1pt minus 1pt,
  partopsep=0pt,
  listparindent=\parindent
}

\begin{document}
\begin{titlepage}

\begin{flushright}
Last Update: April 5, 2026
\end{flushright}

\vskip 2.5em
\begin{center}

{
\LARGE \bfseries %
\begin{spacing}{1.15} %
\input{title} %
\end{spacing}
}

\vskip 1em
{ 
\begin{spacing}{1.5}
Maojiang 
Su$^{\dagger}$\email{smj@u.northerstern.edu}
\quad
Po-Chung 
Hsieh$^{\ddag}$\email{b10505001@ntu.edu.tw}
\quad
Weimin Wu$^{\dagger}$\email{wwm@u.northerstern.edu}
\quad
Mingcheng Lu\email{2860215400pp@gmail.com}
\quad
\\
Jiunhau Chen$^{\sharp}$\email{b13202007@ntu.edu.tw}
\quad
Jerry Yao-Chieh 
Hu$^{\dagger}$\email{jhu@u.northwestern.edu}
\quad
Han Liu$^{\dagger\S}$\email{hanliu@northwestern.edu}
\end{spacing}
}

\def\thefootnote{*}
\footnotetext{
Code release upon acceptance.
}

\vskip 1em

{\small
\begin{tabular}{ll}
 $^\dagger\;$Center for Foundation Models and Generative AI, Northwestern University, Evanston, IL 60208, USA\\
 \hphantom{$^\ddag\;$}Department of Computer Science, Northwestern University, Evanston, IL 60208, USA\\
 $^\ddag\;$Department of Electrical Engineering, National Taiwan University, Taipei 10617, Taiwan\\
 $^\sharp\;$Department of Physics, National Taiwan University, Taipei 10617, Taiwan\\
 $^\S\;$Department of Statistics and Data Science, Northwestern University, Evanston, IL 60208, USA
\end{tabular}}

\end{center}

\noindent
\input{0abstract}

\end{titlepage}

\section{Introduction}
\label{sec:intro}
\input{1intro}

\section{Preliminaries}
\label{sec:background}
\input{2background}

\section{DoMinO: Discrete Flow Matching Policy Optimization}
\label{sec:method}
\input{2method}

\section{Experimental Studies}
\label{sec:exp}\input{3exp}

\section{Discussion and Conclusion}
\label{sec:conclusion}\input{4conclusion}

\section*{Impact Statement}
\input{impact}

\section*{Acknowledgments}
\input{x_acknowledgments}

\newpage
\appendix
\label{sec:append}
\part*{Appendix}
{
\setlength{\parskip}{-0em}
\startcontents[sections]
\printcontents[sections]{ }{1}{}
}

\input{appendix}

\clearpage
\def\arxivfont{\rm}
\bibliographystyle{plainnat}

\bibliography{refs}

\end{document}

%% file: macros.tex
\allowdisplaybreaks
\usepackage{graphicx}
\usepackage{tikz}
\usepackage{tikz-cd}
\usepackage{times}
 
\usepackage{bm}
\usepackage{physics}
\usepackage{xcolor}
\usepackage{natbib}
\usepackage{mdframed}
\usepackage{nicefrac}
\usepackage{booktabs}
\usepackage{lipsum}
\usepackage{titlesec}
\usepackage{wrapfig,lipsum,booktabs}
\usepackage{authblk}
\usepackage{blindtext}
\usepackage[font=small]{caption}

\usepackage{algorithm}
\usepackage{algpseudocode}
\newcommand{\commentsymbol}{//}%
\algrenewcommand\algorithmiccomment[1]{\hfill {\footnotesize \commentsymbol{} #1}}

\usepackage{titletoc}
\usepackage{todonotes}
\usepackage{authblk}
\usepackage{setspace}
\usepackage{dsfont} %

\usepackage[nameinlink,capitalize,noabbrev]{cleveref}

\usepackage{CJK}

\usepackage{array}
\usepackage{bbm}
\usepackage{makecell}

\usepackage{dashbox}
\usepackage{xcolor}
\usepackage{colortbl}
\definecolor{lightyellow}{rgb}{1.0, 0.95, 0.7}
\definecolor{Blue}{rgb}{0, 0, 0.8}
\definecolor{blue}{rgb}{0,0,1}
\definecolor{darkgreen}{rgb}{0,0.40,0}
\definecolor{firebrick}{rgb}{0.698,0.133,0.133}

\definecolor{colorA}{rgb}{1,0,0}
\definecolor{colorB}{rgb}{0,0.3,1}
\definecolor{colorC}{rgb}{0.9,0.8,0.2}
\definecolor{colorD}{rgb}{0,0.65,0}
\definecolor{lesslightgray}{rgb}{0.5,0.5,0.5}
\definecolor{light-gray}{gray}{0.95}

\let\tilde\widetilde
\let\hat\widehat

\def\R{\mathbb{R}}

\DeclareMathOperator*{\E}{{\mathbb{E}}} %

\let\cite\citep 

\usepackage{amsthm}
\usepackage[many]{tcolorbox}
\usepackage{etoolbox}
\BeforeBeginEnvironment{proof}{\par\vspace{-1\baselineskip}}
\AfterEndEnvironment  {proof}{\par\vspace{-1.5\baselineskip}}

\newtheoremstyle{theoremstyle}
  {.5\baselineskip} %
  {.5\baselineskip} %
  {}                  %
  {}                  %
  {\bfseries}        %
  {.}                 %
  {1em}               %
  {}                  %

\theoremstyle{theoremstyle}
\newtheorem{theorem}{Theorem}[section]
\newtheorem{lemma}{Lemma}[section]

\newtheorem{proposition}{Proposition}[section]

\tcolorboxenvironment{theorem}{
  breakable,
  colback=black!10,
  colframe=white,%
  width=\dimexpr\linewidth+10pt\relax,%
  enlarge left by=-5pt,%
  enlarge right by=-5pt,%
  boxsep=5pt,%
  boxrule=0pt,
  left=0pt,right=0pt,top=0pt,bottom=0pt,
  sharp corners,
  before skip=0.5\baselineskip, %
  after skip=0.5\baselineskip,  %
  fonttitle=\bfseries, %
  coltitle=black %
}
\tcolorboxenvironment{remark}{
  blanker,
  breakable,
  before skip=.8\baselineskip,  %
  after  skip=.8\baselineskip   %
}

\tcolorboxenvironment{proposition}{
  breakable,
  colback=black!10,
  colframe=white,%
  width=\dimexpr\linewidth+10pt\relax,%
  enlarge left by=-5pt,%
  enlarge right by=-5pt,%
  boxsep=5pt,%
  boxrule=0pt,
  left=0pt,right=0pt,top=0pt,bottom=0pt,
  sharp corners,
  before skip=0.5\baselineskip, %
  after skip=0.5\baselineskip,  %
  fonttitle=\bfseries, %
  coltitle=black %
}

\tcolorboxenvironment{lemma}{
  breakable,
  colback=black!10,
  colframe=white,%
  width=\dimexpr\linewidth+10pt\relax,%
  enlarge left by=-5pt,%
  enlarge right by=-5pt,%
  boxsep=5pt,%
  boxrule=0pt,
  left=0pt,right=0pt,top=0pt,bottom=0pt,
  sharp corners,
  before skip=0.5\baselineskip, %
  after skip=0.5\baselineskip,  %
  fonttitle=\bfseries, %
  coltitle=black %
}

\tcolorboxenvironment{corollary}{
  breakable,
  colback=black!10,
  colframe=white,%
  width=\dimexpr\linewidth+10pt\relax,%
  enlarge left by=-5pt,%
  enlarge right by=-5pt,%
  boxsep=5pt,%
  boxrule=0pt,
  left=0pt,right=0pt,top=0pt,bottom=0pt,
  sharp corners,
  before skip=0.5\baselineskip, %
  after skip=0.5\baselineskip,  %
  fonttitle=\bfseries, %
  coltitle=black %
}

\tcolorboxenvironment{definition}{
  breakable,
  colback=black!10,
  colframe=white,%
  width=\dimexpr\linewidth+10pt\relax,%
  enlarge left by=-5pt,%
  enlarge right by=-5pt,%
  boxsep=5pt,%
  boxrule=0pt,
  left=0pt,right=0pt,top=0pt,bottom=0pt,
  sharp corners,
  before skip=0.5\baselineskip, %
  after skip=0.5\baselineskip,  %
  fonttitle=\bfseries, %
  coltitle=black %
}
\tcolorboxenvironment{assumption}{
  breakable,
  colback=black!10,
  colframe=white,%
  width=\dimexpr\linewidth+10pt\relax,%
  enlarge left by=-5pt,%
  enlarge right by=-5pt,%
  boxsep=5pt,%
  boxrule=0pt,
  left=0pt,right=0pt,top=0pt,bottom=0pt,
  sharp corners,
  before skip=0.5\baselineskip, %
  after skip=0.5\baselineskip,  %
  fonttitle=\bfseries, %
  coltitle=black %
}

\crefname{theorem}{Theorem}{Theorems}
\crefname{proposition}{Proposition}{Propositions}
\crefname{lemma}{Lemma}{Lemmas}
\crefname{corollary}{Corollary}{Corollaries}
\crefname{definition}{Definition}{Definitions}
\crefname{assumption}{Assumption}{Assumptions}
\crefname{remark}{Remark}{Remarks}
\crefname{problem}{Problem}{Problems}
\crefname{property}{Property}{property}

\tcolorboxenvironment{hypothesis}{
  breakable,
  colback=black!10,
  colframe=white,%
  width=\dimexpr\linewidth+10pt\relax,%
  enlarge left by=-5pt,%
  enlarge right by=-5pt,%
  boxsep=5pt,%
  boxrule=0pt,
  left=0pt,right=0pt,top=0pt,bottom=0pt,
  sharp corners,
  before skip=0.5\baselineskip, %
  after skip=0.5\baselineskip,  %
  fonttitle=\bfseries, %
  coltitle=black %
}
\crefname{hypothesis}{Hypothesis}{Hypothesises}

\tcolorboxenvironment{fact}{
  breakable,
  colback=black!10,
  colframe=white,%
  width=\dimexpr\linewidth+10pt\relax,%
  enlarge left by=-5pt,%
  enlarge right by=-5pt,%
  boxsep=5pt,%
  boxrule=0pt,
  left=0pt,right=0pt,top=0pt,bottom=0pt,
  sharp corners,
  before skip=0.5\baselineskip, %
  after skip=0.5\baselineskip,  %
  fonttitle=\bfseries, %
  coltitle=black %
}
\crefname{fact}{Fact}{Facts}

\tcolorboxenvironment{example}{
  breakable,
  colback=black!10,
  colframe=white,%
  width=\dimexpr\linewidth+10pt\relax,%
  enlarge left by=-5pt,%
  enlarge right by=-5pt,%
  boxsep=5pt,%
  boxrule=0pt,
  left=0pt,right=0pt,top=0pt,bottom=0pt,
  sharp corners,
  before skip=0.5\baselineskip, %
  after skip=0.5\baselineskip,  %
  fonttitle=\bfseries, %
  coltitle=black %
}
\crefname{example}{Example}{Examples}

\tcolorboxenvironment{question}{
  breakable,
  colback=black!10,
  colframe=white,%
  width=\dimexpr\linewidth+10pt\relax,%
  enlarge left by=-5pt,%
  enlarge right by=-5pt,%
  boxsep=5pt,%
  boxrule=0pt,
  left=0pt,right=0pt,top=0pt,bottom=0pt,
  sharp corners,
  before skip=0.5\baselineskip, %
  after skip=0.5\baselineskip,  %
  fonttitle=\bfseries, %
  coltitle=black %
}

\crefname{question}{Question}{Questions}
\numberwithin{equation}{section}
\numberwithin{theorem}{section}
\numberwithin{proposition}{section}
\numberwithin{definition}{section}
\numberwithin{lemma}{section}
\numberwithin{assumption}{section}
\numberwithin{remark}{section}

\usepackage{lipsum}

\newcommand*{\annot}[1]{\tag*{\footnotesize{\textcolor{black!50}{\big(#1\big)}}}}

\makeatletter
\let\save@mathaccent\mathaccent
\newcommand*\if@single[3]{%
    \setbox0\hbox{${\mathaccent"0362{#1}}^H$}%
    \setbox2\hbox{${\mathaccent"0362{\kern0pt#1}}^H$}%
    \ifdim\ht0=\ht2 #3\else #2\fi
}
\newcommand*\rel@kern[1]{\kern#1\dimexpr\macc@kerna}
\newcommand*\widebar[1]{\@ifnextchar^{{\wide@bar{#1}{0}}}{\wide@bar{#1}{1}}}
\newcommand*\wide@bar[2]{\if@single{#1}{\wide@bar@{#1}{#2}{1}}{\wide@bar@{#1}{#2}{2}}}
\newcommand*\wide@bar@[3]{%
    \begingroup
    \def\mathaccent##1##2{%
        \let\mathaccent\save@mathaccent
        \if#32 \let\macc@nucleus\first@char \fi
        \setbox\z@\hbox{$\macc@style{\macc@nucleus}_{}$}%
        \setbox\tw@\hbox{$\macc@style{\macc@nucleus}{}_{}$}%
        \dimen@\wd\tw@
        \advance\dimen@-\wd\z@
        \divide\dimen@ 3
        \@tempdima\wd\tw@
        \advance\@tempdima-\scriptspace
        \divide\@tempdima 10
        \advance\dimen@-\@tempdima
        \ifdim\dimen@>\z@ \dimen@0pt\fi
        \rel@kern{0.6}\kern-\dimen@
        \if#31
        \overline{\rel@kern{-0.6}\kern\dimen@\macc@nucleus\rel@kern{0.4}\kern\dimen@}%
        \advance\dimen@0.4\dimexpr\macc@kerna
        \let\final@kern#2%
        \ifdim\dimen@<\z@ \let\final@kern1\fi
        \if\final@kern1 \kern-\dimen@\fi
        \else
        \overline{\rel@kern{-0.6}\kern\dimen@#1}%
        \fi
    }%
    \macc@depth\@ne
    \let\math@bgroup\@empty \let\math@egroup\macc@set@skewchar
    \mathsurround\z@ \frozen@everymath{\mathgroup\macc@group\relax}%
    \macc@set@skewchar\relax
    \let\mathaccentV\macc@nested@a
    \if#31
    \macc@nested@a\relax111{#1}%
    \else
    \def\gobble@till@marker##1\endmarker{}%
    \futurelet\first@char\gobble@till@marker#1\endmarker
    \ifcat\noexpand\first@char A\else
    \def\first@char{}%
    \fi
    \macc@nested@a\relax111{\first@char}%
    \fi
    \endgroup
    }
\makeatother

%% file: title.tex
Discrete Flow Matching Policy Optimization 

%% file: 0abstract.tex
We introduce \underline{D}iscrete fl\underline{o}w \underline{M}atch\underline{in}g policy \underline{O}ptimization (DoMinO), a unified framework for Reinforcement Learning (RL) fine-tuning Discrete Flow Matching (DFM) models under a broad class of policy gradient methods. 
Our key idea is to view the DFM sampling procedure as a multi-step Markov Decision Process.
This perspective provides a simple and transparent reformulation of fine-tuning reward maximization as a robust RL objective.
Consequently, it not only preserves the original DFM samplers but also avoids biased auxiliary estimators and likelihood surrogates used by many prior RL fine-tuning methods. 
To prevent policy collapse, we also introduce new total-variation regularizers to keep the fine-tuned distribution close to the pretrained one. 
Theoretically, we establish an upper bound on the discretization error of DoMinO and tractable upper bounds for the regularizers.
Experimentally, we evaluate DoMinO on regulatory DNA sequence design. 
DoMinO achieves stronger predicted enhancer activity and better sequence naturalness than the previous best reward-driven baselines.
The regularization further improves alignment with the natural sequence distribution while preserving strong functional performance.
These results establish DoMinO as an useful framework for controllable discrete sequence generation.

%% file: 1intro.tex
We introduce \underline{D}iscrete fl\underline{o}w \underline{M}atch\underline{in}g policy \underline{O}ptimization (DoMinO), a unified Reinforcement Learning (RL) fine-tuning framework for Discrete Flow Matching (DFM) generative models.
Methodologically, this DFM fine-tuning framework supports many popular policy-gradient methods, including REINFORCE \cite{williams1992simple}, PPO \cite{schulman2017proximal}, and GRPO \cite{shao2024deepseekmath}.
Theoretically,  this DFM fine-tuning framework possess nice guarantees for the discretization error and the total-variation distance regularizations.
Experimentally, we show that DoMinO achieves state-of-the-art performance on regulatory DNA sequence design.

Discrete Flow Matching \cite{campbell2024generative,gat2024discrete,shaul2024flow} is an effective framework for discrete generative modeling, with strong results in speech recognition \cite{navon2025drax}, graph generation \cite{qin2024defog}, video generation \cite{fuest2025maskflow,deng2025uniform}, and biological sequence modeling \cite{yi2025all,gat2024discrete}.
Compared with discrete diffusion models \cite{campbell2022continuous,sun2022score}, discrete flow matching directly parameterizes the transition rates (velocity) of a Continuous-Time Markov Chain (CTMC). 
This allows a more flexible design space for path parameterization and sampling strategies \cite{lipman2024flow}.
However, despite the success of these DFM pretrained models, their reward-driven fine-tuning remains underexplored, even though recent work shows its value for pretrained discrete generative models \cite{zekri2025fine,wang2024fine,zhao2025d1}.

To establish effective reward-driven fine-tuning methods for DFM models,
a natural objective is to maximize the expected reward of terminal samples.
However, this objective faces three challenges in the discrete flow matching setting.
First, DFM parameterizes the transition rates of an underlying CTMC rather than the policy itself, so the exact policy likelihood is not tractable.
Second, many reward functions are non-differentiable, which prevents direct optimization through the reward.
Third, reward optimization may cause over-optimization and push the model away from the pretrained distribution.
In many applications, such as DNA sequence design, we want the generated samples to remain natural, which means staying close to the pretrained model.

To address the first two challenges, we adopt policy gradient methods, which do not require differentiable rewards.
Our key idea is to reinterpret the DFM sampling procedure as an inner multi-step Markov Decision Process (MDP).
This view turns terminal reward maximization into a standard RL objective.
Crucially, the one-step policy coincides with the jump distribution, whose log-likelihood remains tractable.
This structure enables direct application of stable and efficient policy-gradient methods such as REINFORCE \cite{williams1992simple}, PPO \cite{schulman2017proximal}, and GRPO \cite{shao2024deepseekmath}.
To address the third challenge, we further introduce two total-variation distance regularizers.
These regularizers are better aligned with sample-level naturalness than path-wise KL regularization \cite{wang2024fine,zekri2025fine,rojas2025improving}.

\paragraph{Contributions.}
Our contributions are three-fold:
\begin{itemize}
   \item  \textbf{Method:}
    We propose \underline{D}iscrete fl\underline{o}w \underline{M}atch\underline{in}g policy \underline{O}ptimization  (DoMinO), a reinforcement learning fine-tuning framework for discrete flow matching models (\cref{sec:method}).
    We reframe DFM inference as an inner MDP whose policy is exactly the DFM one-step transition kernel (\cref{sec:inner_mdp}). 
    It makes the exact log-likelihood tractable and enables direct policy gradient optimization.
    Under this framework, \cref{sec:pg_dfm} derives DoMinO-REINFORCE (\cref{alg:DoMinO_reinforce}) and DoMinO-PPO (\cref{alg:DoMinO_ppo}).
    DoMinO supports non-differentiable rewards without additional approximation, steering, or guidance mechanisms, and extends naturally to conditional generation.
    To control over-optimization and preserve naturalness, we further introduce a cross-entropy regularizer and a generalized KL regularizer that keep the fine-tuned model close to the pretrained reference model (\cref{sec:tv_regularization}).
    
    \item \textbf{Theory:}
    We provide theoretical justifications for DoMinO and the proposed Total Variation distance regularizations (\cref{sec:theory}).
    Specifically, we analyze the discretization error of reward fine-tuning under the Euler sampler (\cref{thm:error_bound_euler_step_maintex}) and show that both the expected reward and its gradient incur only $O(\Delta t)$ numerical error.
    We also derive upper bounds on the terminal Total Variation distance: a Cross-Entropy bound (\cref{thm:regularization_ce_maintex}) and a generalized KL bound (\cref{thm:regularization_general_kl_maintex}).
    These results justify our regularization designs by showing that the proposed regularizers control distributional drift from the reference model.

    \item \textbf{Experiment:}
    We validate DoMinO on the regulatory DNA sequence design task (\cref{sec:exp}).
    DoMinO achieves stronger predicted enhancer activity and better sequence naturalness than the previous state-of-the-art baselines, e.g., DRAKES \cite{wang2024fine} and SEPO \cite{zekri2025fine}.
    With proposed regularization, it further improves alignment with the natural sequence distribution while preserving strong functional performance.
   
\end{itemize}
These results demonstrate that policy-gradient fine-tuning of DFM provides an effective framework for controllable discrete sequence generation.
It offers a better trade-off between functional optimization and sequence naturalness than diffusion-based or prior reward-driven methods.

\section{Related Works}
\textbf{Discrete Flow Matching.}
Recent progress in discrete generative modeling moves beyond autoregressive models toward continuous-time formulations on discrete state spaces, most notably discrete diffusion and CTMC-based denoising frameworks \cite{hoogeboom2021argmax, sun2022score, campbell2022continuous,austin2023structureddenoisingdiffusionmodels}.
Within this line, Discrete flow matching (DFM) provides a flow-based view of discrete generation, generalizing diffusion-style constructions while allowing  greater flexibility in the choice of probability paths and transition dynamics \cite{campbell2024generative,gat2024discrete}.
Subsequently, DFM develop into a useful framework for structured discrete domains \cite{gat2024discrete,shaul2024flow}.
Applications of DFM now cover several structured discrete generation problems.
For instance, its applications expand to visual token generation, with MaskFlow enabling efficient long-horizon video synthesis through discrete flow-based modeling \cite{fuest2025maskflow}.
Further, for biomolecular design Generative Flows on Discrete State-Spaces studies multimodal protein co-design  \cite{campbell2024generative}, while ADFLIP extends DFM to all-atom inverse protein folding \cite{yi2025all}.
However, the existing DFM research still centers on pretraining and generative modeling, whereas reward-driven post-training is explored primarily for discrete diffusion models rather than DFM itself \cite{wang2024fine,zekri2025fine}.
Our work addresses this gap by studying policy optimization for DFM directly, rather than adapting techniques developed for diffusion-based discrete generators.

\textbf{RL for Discrete Generative Models.}
Prior work on reinforcement learning for discrete generative models focus mainly on discrete diffusion models.
DRAKES \cite{wang2024fine} applies reinforcement learning to discrete diffusion through Direct Reward backpropagation with the Gumbel-Softmax trick. This design restricts the method to continuous reward signals.
Score Entropy Policy Optimization (SEPO) \cite{zekri2025fine} studies policy gradient fine-tuning for discrete diffusion models. However, it relies on self-normalized importance sampling (SNIS) for additional estimation.
\cite{zhao2025d1} develops a reinforcement learning algorithm for discrete diffusion, but it uses an approximation that does not yield an unbiased estimator.
\cite{nower2026flow} studies reinforcement learning for discrete flow matching through a reward-reweighted conditional flow matching loss.
In contrast, our work develops stable and efficient policy gradient methods for reinforcement learning in discrete flow matching. Our method is unbiased, does not require additional estimation, and does not rely on continuous rewards.

%% file: 2background.tex
In this section, we provide an high level review of discrete flow matching following \cite{lipman2024flow,su2025theoretical}, and the reinforcement learning.

\paragraph{Continuous-Time Markov Chain.}
Consider discrete data 
$x$ taking values in the state space 
$S= \mathcal{V}^d$ where vocabulary $ \mathcal{V} = \{1, \ldots, M\}$.
The Continuous-Time Markov Chain (CTMC)  \cite{norris1998markov} is a continuous stochastic process $(X_t)_{t \geq 0}$ on $S$ that satisfies the Markov property, which the system's future state depends only on the current state, not on the past history.
Let $p_t$ denote the Probability 
Mass 
Function (PMF) of $X_t$.
Then we define an unique Continuous-Time Markov Chain by specifying an initial distribution $p_0$
and rates function (velocity field) $u_t(y,x): S \times S \to \R $. 
This rates function induces the probability transition kernel 
$p_{t+\Delta t|t}$,
\begin{align}
\label{eqn:transition_kernel}
    p_{t+\Delta t|t}(y|x) := P(X_{t+\Delta t}=y|X_t=x)
    =  \delta(y,x) + u_t(y,x) \Delta t + o(\Delta t),
\end{align}
where $\delta(y,x)$ is the Kronecker delta function, equal to $1$ when $x=y$ and $0$ otherwise.
The values $u_t(y,x)$ represent the instantaneous rate of transition from state $x$ to state $y$ at time $t$.
We define $u_t$ \textit{generates} $p_t$ if there exists transition kernels $p_{t+\Delta t|t}$ satisfying \eqref{eqn:transition_kernel} whose induced marginal PMFs are $(p_t)_{t \geq 0}$.
For $p_{t+\Delta t|t}(\cdot |x)$ to be a valid probability mass function, i.e., $\sum_y p_{t+\Delta t|t}(y|x) = 1$, the rates function $u_t(y,x)$ must satisfy the following  rates conditions: $\text{for all}~~ y \neq x$,
\begin{align}
\label{eqn:rates_condition}
    u_t(y,x) \geq 0, \quad \text{and} \quad \sum_{y \in S} u_t(y,x)= 0.
\end{align}
By the definition of transition kernel \eqref{eqn:transition_kernel}, a rates function $u_t$ and an initial distribution $p_0$ define a unique probability path $p_t$ via the Kolmogorov Equation \cite[Theorem 12]{lipman2024flow}, 
\begin{align}
\label{eq:kolmogorov}
    \dv{p_t(y)}{t} = \sum_{x \in S} u_t(y,x) p_t(x).
\end{align}
To sample $X_T$,  we sample $X_0 \sim p_0$ and simulate sample trajectory with (naive) Euler method
\begin{align}
\label{eq:euler_method}
    P(X_{t+\Delta t}=y|X_t=x)
    = \delta(y,x) + u_t(y,x) \Delta t, \quad
    \text{with} \quad P(X_0)=  p_0(x).
\end{align}

\paragraph{Discrete Flow Matching.}
Discrete Flow Matching (DFM) is a generative modeling framework that learns a transformation from a source distribution $p_0$ to a target distribution $p_T$ \cite{campbell2024generative,gat2024discrete}. 
The key idea is to construct a velocity $u_t$ that induce a probability path $(p_t)_{t \in [0, T]}$  interpolates between $p_0$ and $p_T$. 
The learning objective is to train a neural network $u_t^\theta$ to approximate this ground-truth velocity $u_t$.
We train the model by minimizing the discrete flow matching loss with a Bregman divergence $D(\cdot,\cdot)$ (see \cref{sec:intro} for definition)
\begin{align*}
    \mathcal{L}_{\text{DFM}} = \mathbb{E}_{t, X_t \sim p_t} \left[ D(u_t(\cdot, X_t), u_t^\theta(\cdot, X_t)) \right],
\end{align*}
where the ground-truth velocity $u_t(\cdot, X_t)$  satisfies the rate conditions in \eqref{eqn:rates_condition}.
However, the ground-true velocity $u_t(\cdot,\cdot)$ is intractable.
Conditional Discrete Flow Matching (CDFM) \cite{campbell2024generative, gat2024discrete} provides a tractable loss, 
\begin{align*}
    \mathcal{L}_{\text{CDFM}} = \mathbb{E}_{t, Z \sim p_Z, X_t \sim p_{t|Z}} \left[ D(u_t(\cdot,X_t|Z), u_t^\theta(\cdot,X_t)) \right].
\end{align*}
Crucially, the CDFM and DFM objectives yield identical learning gradients \cite[Theorem 15]{lipman2024flow}, that is,
$\nabla_{\theta} \mathcal{L}_{\text{CDFM}}(\theta) = \nabla_{\theta} \mathcal{L}_{\text{DFM}}(\theta)$.
We use the standard mixture path and parameterize the DFM model through the token-wise posterior distributions $p_{1|t}^\theta(\cdot|x)$. In training, we use the generalized KL divergence as the Bregman divergence $D(\cdot,\cdot)$ following \cite{lipman2024flow}.

\textbf{Reinforcement Learning.}
We consider a Markov 
Decision 
Process (MDP) with a tuple $(\mathcal{S}, \mathcal{A}, P, R, \gamma)$.
Here, $\mathcal{S}$ and $\mathcal{A}$ denote the state and action spaces, 
$P(s' | s,a)$ is the transition kernel,
$R(s,a)$ is the reward function,
and $\gamma \in [0,1)$ is the discount factor.
A policy $\pi(a | s)$ maps each state to a distribution over actions.
Given policy $\pi$, it induces a distribution over trajectories
$\tau = (s_0,a_0,s_1,a_1,\ldots)$, where $a_t \sim \pi(\cdot | s_t)$
and $s_{t+1} \sim P(\cdot | s_t,a_t)$.
Given a trajectory $\tau$, we define the discounted return as
$G(\tau) = \sum_{t \ge 0} \gamma^t R(s_t,a_t)$.
The objective of RL is to learn parameters $\theta$ that maximize the expected return under the induced trajectory distribution,
\begin{align}
\label{eq:rl_objective_prelim}
    J_{\text{RL}}(\theta)
    =
    \E_{\tau \sim \pi_\theta}
    \Big[
    \sum_{t \ge 0} \gamma^t R(s_t,a_t)
    \Big].
\end{align}
Reinforcement learning algorithms often optimize \eqref{eq:rl_objective_prelim} with policy gradient methods \cite{williams1992simple,schulman2015trust,schulman2017proximal,shao2024deepseekmath}.
These methods estimate the gradient of $J_{\text{RL}}(\theta)$ and optimize the policy parameters $\theta$ with gradient ascent.
Compared with value-based methods, policy gradient methods directly optimize the policy parameters.
Therefore, they avoid an explicit maximization over actions and makes them 
suitable for continuous action spaces.

%% file: 2method.tex
In this section, we propose \underline{D}iscrete fl\underline{o}w \underline{M}atch\underline{in}g policy \underline{O}ptimization (DoMinO), a RL fine-tuning framework for DFM.
Specifically, \cref{sec:prob_stat} formulates the reward fine-tuning problem, \cref{sec:inner_mdp} reinterprets DFM inference as an inner multi-step Markov Decision Process (MDP) and reformulates reward maximization as a standard RL objective, and \cref{sec:pg_dfm} instantiates DoMinO with REINFORCE \cite{williams1992simple} and PPO \cite{schulman2017proximal}.

\subsection{Problem Statement}
\label{sec:prob_stat}
Assume there is a pre-existing discrete flow matching model $u_t^\theta(y,x)$, either pretrained or randomly initialized.
Let $p_T^\theta$ denote the terminal distribution induced by $u_t^\theta(y,x)$.
We study the problem of fine-tuning this discrete flow matching model to maximize the expected reward
\begin{align}
\label{eq:RL_loss}
    J(\theta)
    =
    \E_{X_T \sim p_T^\theta}[r(X_T)].
\end{align}
The objective in \eqref{eq:RL_loss} is simple.
However, discrete flow matching does not parameterize the  explicit marginal distribution $p_T^\theta$.
Instead, it parameterizes the transition velocity $u_t^\theta(y,x)$.
As a result, directly optimizing \eqref{eq:RL_loss} is difficult.
It requires either differentiating through the sampling process or estimating likelihood ratios.
Both approaches are nontrivial and often suffer from high variance or expensive marginal estimation.
To address this difficulty, we reformulate \eqref{eq:RL_loss} as a standard reinforcement learning objective defined on an inner MDP in the next subsection.

\subsection{Reframe Inference as an Inner MDP}
\label{sec:inner_mdp}

We first reframe the inference process of discrete flow matching as a multi-step inner MDP.
We then show that the standard reinforcement learning objective on this inner MDP is equivalent to the original objective $J(\theta)$.
Our method draws inspiration from denoising diffusion policy optimization \cite{black2023training}, which casts the denoising process of diffusion models as an inner MDP.
We extend this idea to the discrete setting with discrete flow matching.

Given the transition velocity $u_t^\theta(y,x)$, the terminal reward function $r$, and the time step $\Delta t$, we define the corresponding inner MDP
$\mathcal{M}=(\mathcal{S},\mathcal{A},P,R,1)$ as
\begin{align}
\label{eq:inner_mdp}
    s_t = & ~(t,x_t), \quad \pi_t^\theta(a_t | s_t) = p_t^\theta(x_{t+\Delta t} | x_t), \quad
    P(s_{t+\Delta t} | s_t,a_t) = (\delta_{t+\Delta t}, \delta_{x_{t+\Delta t}})
    \nonumber\\
    a_t = & ~x_{t+\Delta t},
    \quad P_0 = (\delta_0, p_0), \quad
    R(s_t,a_t) =
    \begin{cases}
        r(x_T), & t = T-\Delta t;\\
        0, & \text{otherwise},
    \end{cases}
\end{align}
where $p_0$ is the source distribution of the continuous-time Markov chain, and $\delta_z$ denotes the Dirac delta distribution concentrated at $z$.
The policy $p_t^\theta(x_{t+\Delta t} | x_t)$ takes the form
\begin{align}
\label{eq:one_step_prob}
    p_t^\theta(x_{t+\Delta t} | x_t)
    =
    \begin{cases}
        u_t^\theta(x_{t+\Delta t},x_t)\Delta t, & x_{t+\Delta t} \neq x_t;\\
        1 - \sum_{y \neq x_t} u_t^\theta(y,x_t)\Delta t, & x_{t+\Delta t} = x_t.
    \end{cases}
\end{align}
We call $\mathcal{M}$ an \emph{inner} MDP because it describes a single inference process of the discrete flow matching model.
This differs from a standard environment MDP, where the agent interacts with an external environment that has its own transition dynamics, as in robotics.
By the definition of the inner MDP $\mathcal{M}$ and the reinforcement learning objective $J_{\mathrm{RL}}(\theta)$, we have following proposition.
\begin{proposition}
\label{prop:equivalence}
    Let the fine-tuning objective $J(\theta)$ be defined in \eqref{eq:RL_loss}.
    Let the reinforcement learning objective $J_{\mathrm{RL}}(\theta)$ be defined in \eqref{eq:rl_objective_prelim} on the inner MDP $\mathcal{M}$.
    Then,
    \begin{align*}
        J_{\mathrm{RL}}(\theta) = J(\theta).
    \end{align*}
\end{proposition}

\cref{prop:equivalence} reformulates the original objective $J(\theta)$ as a standard reinforcement learning objective $J_{\mathrm{RL}}$ on the inner MDP $\mathcal{M}$. A key benefit of this equivalence is that the policy $\pi_t^\theta(a_t | s_t)$ in the inner MDP is exactly the one-step transition probability $p_t^\theta(x_{t+\Delta t} | x_t)$, which is tractable with \eqref{eq:one_step_prob}. This makes $J_{\mathrm{RL}}$ natural to optimize with stable and efficient policy gradient methods.

\textbf{Conditional Generation.}
Our framework extends directly to conditional generation. Assume each sample is associated with a condition $c$, and let $u_t^\theta(y,x |c)$ denote a pretrained conditional discrete flow matching model. We then condition all quantities in the inner MDP on $c$. In particular, the DFM transition kernel becomes $p_t^\theta(x_{t+\Delta t} |x_t,c)$, the induced policy becomes $\pi_t^\theta(a_t |s_t,c)$, the source distribution becomes $p_0(\cdot |c)$, and the terminal reward becomes $r(x_T,c)$. Since the condition remains fixed throughout the sampling trajectory, the same inner-MDP construction and the same policy gradient methods apply without modification. In particular, this setting is well suited to GRPO \cite{shao2024deepseekmath}, which is widely used for RL in conditional generation.

\subsection{Discrete Flow Matching Policy Optimization}
\label{sec:pg_dfm}

In this section, we optimize the reinforcement learning objective $J_{\mathrm{RL}}(\theta)$ with policy gradient methods. 
We refer to this class of algorithms as \underline{D}iscrete fl\underline{o}w \underline{M}atch\underline{in}g policy \underline{O}ptimization (DoMinO). Below, we present two instantiations based on different gradient estimators.

\textbf{Discrete Flow Matching Policy Optimization with REINFORCE.}
We optimize $J_{\mathrm{RL}}(\theta)$ with the REINFORCE algorithm \cite{williams1992simple}.
REINFORCE uses a  log-likelihood  gradient estimator.
Concretely, REINFORCE estimates the policy gradient $\nabla_\theta J_{\mathrm{RL}}(\theta)$ by
\begin{align}
\label{eq:reinforce_grad}
    \nabla_\theta J_{\mathrm{RL}}(\theta)
    =
    \E_{\tau\sim \pi_\theta}
    \left[
        \sum_{t=0}^{T-\Delta t}
        \nabla_\theta \log \pi^\theta_t(a_t| s_t)\cdot r(x_T)
    \right].
\end{align}
Following the definition of inner MDP $\mathcal{M}=(\mathcal{S},\mathcal{A},P,R,1)$ in \eqref{eq:inner_mdp}, the policy at time $t$ equals the one-step transition of the discretized CTMC,
$\pi_t^\theta(\cdot|s_t)=p_t^\theta(\cdot|x_t)$.
Therefore, $\log \pi^\theta_t(a_t| s_t)=\log p_t^\theta(x_{t+\Delta t}|x_t)$ is tractable through \eqref{eq:one_step_prob}.
Here, $\tau$ denotes the discrete flow matching inference trajectory $\tau=
\{x_t\}_{t=0}^T$ generated by the inference process \eqref{eq:euler_method} with the rate model $u_t^\theta$, and $r(x_T)$ is the terminal reward.
In practice, we replace raw terminal reward $r(x_T)$ with an advantage $\hat A(x_T)$ to reduce variance.
We summarize the training procedure in \cref{alg:DoMinO_reinforce}.

\begin{algorithm}[t]
\caption{DoMinO-REINFORCE}
\label{alg:DoMinO_reinforce}
\begin{algorithmic}[1]
\Require Pre-trained rate model $u_\phi$; step size $\Delta t$; horizon $T$;
batch size $M$; learning rate $\eta$; iterations $K$;
terminal reward function $r(x)$; Advantage estimator $\hat A$
\State \textbf{Initialize:} $\theta \leftarrow \phi$
\For{training iteration $k=1,2,\ldots,K$}
    \For{$m=1,2,\ldots,M$}
    \State Sample trajectory
    $\tau^{(m)}=
    \{x_t^{(m)}\}_{t=0}^T$ by discrete flow matching inference \eqref{eq:euler_method} with $u_t^\theta$
    \State Compute rewards for the trajectory
    $R^{(m)} \leftarrow r(x_T^{(m)})$
    \EndFor
    \State Estimate advantages 
    $\hat A^{(m)} \leftarrow \hat 
    A\left(\{\tau^{(m)},
    R^{(m)}\}
    _{m=1}^M\right)$ 
    for $m \in [M]$
    \State Policy gradient update:
    \begin{align*}
        \theta \leftarrow \theta + \eta
        \frac{1}{M}\sum_{m=1}^M
        \sum_{t=0}^{T-\Delta t}
        \nabla_\theta \log p_t^\theta(x_{t+\Delta t}^{(m)}|x_t^{(m)})
        \hat A^{(m)}
    \end{align*}
\EndFor
\State \textbf{Output:} fine-tuned policy $\pi_\theta$ with rate model $u_t^\theta$.
\end{algorithmic}
\end{algorithm}

\textbf{Discrete Flow Matching Policy Optimization with PPO.}
We optimize $J_{\text{RL}}(\theta)$ with Proximal Policy Optimization (PPO) \cite{schulman2017proximal}.
PPO controls the update size by clipping the likelihood ratio between the new policy and the old policy, which improves stability under high-variance rewards and non-differentiable terminal objectives.
We sample trajectories from the old policy $\pi_{\text{old}}$ and update $\theta$ by maximizing the clipped  surrogate objective
\begin{align*}
    J_{\text{PPO}}(\theta)
    =
    \E_{\tau \sim \pi_{\text{old}}}
    \left[
        \sum_{t=0}^{T-\Delta t}
        \min\left\{
            r_t(\theta)A_t,
            \text{Clip}\left(r_t(\theta),1-\epsilon,1+\epsilon\right)
            A_t
        \right\}
    \right],
\end{align*}
where $\tau=\{x_t\}_{t=0}^T$ denotes a discrete flow matching inference trajectory, $a_t := x_{t+\Delta t}$,
$r_t(\theta)=\frac{\pi_\theta(a_t|s_t)}{\pi_{\text{old}}(a_t|s_t)}$ is the likelihood ratio, and $\epsilon$ is the clip parameter.
Following the same inner MDP definition \eqref{eq:inner_mdp}, we again have $\pi_t^\theta(\cdot|s_t)=p_t^\theta(\cdot|x_t)$, so the likelihood ratio $r_t(\theta)$ is tractable with \eqref{eq:one_step_prob}.
In the terminal-reward setting, we often use a trajectory-level advantage and share it across time steps, i.e., $\hat A_t := \hat A$.
We summarize the training procedure in \cref{alg:DoMinO_ppo}.

\begin{algorithm}[t]
\caption{DoMinO-PPO}
\label{alg:DoMinO_ppo}
\begin{algorithmic}[1]
\Require Pre-trained rate model $u_\phi$; step size $\Delta t$; horizon $T$;
batch size $M$; learning rate $\eta$; iterations $K$;
clip parameter $\epsilon$; terminal reward function $r(x)$; advantage estimator $\hat A$
\State \textbf{Initialize:} $\theta \leftarrow \phi$
\For{training iteration $k=1,2,\ldots,K$}
    \State Set $\theta_{\text{old}} \leftarrow \theta$
    \For{$m=1,2,\ldots,M$}
        \State Sample trajectory $\{x_t^{(m)}\}_{t=0}^T$ by discrete flow matching inference \eqref{eq:euler_method} with $u_t^{\theta_{\text{old}}}$.
        \State Compute terminal reward $R^{(m)} \leftarrow r(x_T^{(m)})$
    \EndFor
    \State Estimate advantages $\{\hat A_t^{(m)}\} \leftarrow \hat A(\{\tau^{(m)},R^{(m)}\}_{m=1}^M)$
    \State Compute the ratios
    $r_t^{(m)}(\theta) ={p_t^\theta(x_{t+\Delta t}^{(m)}|x_t^{(m)})}/{p_t^{\theta_{\text{old}}}(x_{t+\Delta t}^{(m)}|x_t^{(m)})}$
    \State PPO update:
    \begin{align*}
        \theta \leftarrow \theta + \eta \nabla_\theta
        \frac{1}{M}\sum_{m=1}^M
        \sum_{t=0}^{T-\Delta t}
        \min\left\{
            r_t^{(m)}(\theta)\hat A_t^{(m)},
            \text{Clip}\left(r_t^{(m)}(\theta), 
            1-\epsilon, 1+\epsilon\right)\hat A_t^{(m)}
        \right\}
    \end{align*}
\EndFor
\State \textbf{Output:} fine-tuned policy $\pi_\theta$ with rate model $u_t^\theta$.
\end{algorithmic}
\end{algorithm}

\section{Total Variation Distance Regularization}
\label{sec:tv_regularization}

We introduce Total Variation (TV) distance regularization to prevent over-optimization in reward-driven fine-tuning and preserve sequence naturalness. Optimizing only the expected terminal reward may push the model toward unrealistic samples that exploit imperfections of the reward function and drift away from the pretrained distribution. 
To avoid this failure mode, a common choice in prior work is path-wise Kullback-Leibler (KL) regularization
\cite{wang2024fine,zekri2025fine,rojas2025improving}.
Let $\mathbb{P}_\theta$ and $\mathbb{P}_{\mathrm{ref}}$ denote the path measures induced by the fine-tuned model and the reference model, respectively, and assume they share the same initial distribution $p_0$. For continuous-time Markov chains, the path-wise KL divergence admits 
\begin{align*}
\mathrm{KL}\left(\mathbb P_\theta \,\|\, \mathbb P_{\mathrm{ref}}\right)
=
\mathbb E_{\mathbb P_\theta}\left[
\int_0^T
\sum_{y \neq X_t}
\left(
u_t^\theta(y,X_t)
\log \frac{u_t^\theta(y,X_t)}{u_t^{\mathrm{ref}}(y,X_t)}
- u_t^\theta(y,X_t)
+ u_t^{\mathrm{ref}}(y,X_t)
\right)
\dd t
\right].
\end{align*}
However, this path-wise regularization has several drawbacks. First, it requires integrating rate-level terms along the full sampling trajectory and over all possible next states, which increases computational cost when the trajectory is long or the per-step state space is large. 
Second, it regularizes the entire trajectory rather than the terminal distribution, which may impose unnecessary constraints when naturalness is defined at the sample level.

We control the shift of the terminal distribution through TV distance. 
It avoids the unnecessary  constraint imposed by path-wise KL regularization.
Let $p_T^\theta$ denote the terminal distribution induced by the fine-tuned model $u_t^\theta$, and let $p_T^{\mathrm{ref}}$ denote a reference terminal distribution induced by reference model $u_t^{\text{ref}}$
(e.g. the pretrained model). 
We define the TV distance on the space $S$ as
\begin{align*}
    \mathrm{TV}(p_T^\theta,p_T^{\mathrm{ref}})
    :=
    \frac{1}{2}
    \sum_{x \in S}
    |p_T^\theta(x)-p_T^{\mathrm{ref}}(x)|.
\end{align*}
Directly optimizing or estimating this terminal TV distance remains intractable. 
Following prior work \cite{gat2024discrete,lipman2024flow}, we use the mixture path and parameterize the DFM model through the  posterior $p_{1|t}^\theta(\cdot|x)$. Under this parameterization, we consider two tractable regularizers. 
The first acts on the posterior and takes the form of a cross-entropy loss. 
The second is induced by the generalized KL divergence used in DFM pretraining (\cref{sec:background}).

\textbf{Cross-Entropy Regularization.}
Let $\theta_{\mathrm{ref}}$ denote the frozen reference parameter, and let $p_t^{\theta_{\mathrm{ref}}}$ be the marginal probability distribution induced by the reference model. 
Under the posterior parameterization, we regularize the fine-tuned posterior toward the reference posterior at states sampled from the reference trajectory. We define the cross-entropy regularizer as
\begin{align}
\mathcal{L}_{\mathrm{reg}}^{\mathrm{CE}}(\theta;\theta_{\mathrm{ref}})
:=
\mathbb{E}_{t, X_t \sim p_t^{\theta_{\mathrm{ref}}}}
\left[
-
\sum_{y \in \mathcal{S}}
p_{1|t}^{\theta_{\mathrm{ref}}}(y|X_t)
\log p_{1|t}^{\theta}(y|X_t)
\right].
\label{eq:dfm_ce_reg}
\end{align}

\textbf{Generalized KL Regularization.}
We also consider a regularizer induced by the generalized KL divergence between the reference velocity $u_t^{\theta_{\mathrm{ref}}}(\cdot, X_t)$ and fine-tuned velocities $u_t^\theta(\cdot, X_t)$. 
For nonnegative vectors $u$ and $v$, define the generalized KL divergence as
\begin{align}
D_{\mathrm{gKL}}(u,v)
:=
\sum_j u_j \log \frac{u_j}{v_j}
-
\sum_j u_j
+
\sum_j v_j.
\label{eq:gkl_def}
\end{align}
We define the generalized KL regularizer is
\begin{align}
\label{eq:dfm_gkl_reg}
\mathcal{L}_{\mathrm{reg}}^{\mathrm{gKL}}(\theta;\theta_{\mathrm{ref}})
:=
\mathbb{E}_{t, X_t \sim p_t^{\theta_{\mathrm{ref}}}}
\left[
D_{\mathrm{gKL}}
\left(
u_t^{\theta_{\mathrm{ref}}}(\cdot, X_t),
u_t^{\theta}(\cdot, X_t)
\right)
\right].
\end{align}
Both regularizers support efficient computation on the same rollouts used by the RL update. In particular, we reuse the stored rollout states $(X_t,t)$ without additional sampling. The regularizer only requires one extra forward pass through the reference model to evaluate the reference posterior or velocity at $X_t$.
In the next section, we show that both $\mathcal{L}_{\mathrm{reg}}^{\mathrm{CE}}$ and $\mathcal{L}_{\mathrm{reg}}^{\mathrm{gKL}}$ provide tractable upper bounds on the terminal TV distance. The full fine-tuning objective is
\begin{align}
J_{\mathrm{total}}(\theta)
=
J_{\mathrm{RL}}(\theta)
-
\lambda \mathcal{L}_{\mathrm{reg}}(\theta;\theta_{\mathrm{ref}}),
\label{eq:full_rl_reg_obj}
\end{align}
where $\mathcal{L}_{\mathrm{reg}}$ is either cross-entropy regularizer $\mathcal{L}_{\mathrm{reg}}^{\mathrm{CE}}$ or generalized KL regularizer $\mathcal{L}_{\mathrm{reg}}^{\mathrm{gKL}}$.

\section{Theoretical Analysis}
\label{sec:theory}
In this section, we provide theoretical justification for DoMinO from two aspects.
First, we analyze the discretization error of RL fine-tuning with the Euler sampler (\cref{thm:error_bound_euler_step_maintex}).
Second, we derive upper bounds on the terminal Total Variation distance in terms of the cross-entropy loss (\cref{thm:regularization_ce_maintex}) and the generalized KL loss (\cref{thm:regularization_general_kl_maintex}).
These results justify the proposed cross-entropy regularizer $\mathcal{L}_{\mathrm{reg}}^{\mathrm{CE}}$ and generalized KL regularizer $\mathcal{L}_{\mathrm{reg}}^{\mathrm{gKL}}$, since they show that both regularizers control distributional drift from the reference model.

\subsection{Discretization Error of RL Fine-Tuning}
\label{subsec:discretion_eroor_of_reward}

We analyze the discretization error in RL fine-tuning under the Euler sampler.
Since DoMinO defines the reward fine-tuning objective through discretized sampling trajectories, we need to understand how this objective differs from its exact continuous-time counterpart.
The following theorem shows that both the reward objective and its gradient incur only first-order error.

\begin{theorem}[Discretization Error of Euler Method]\label{thm:error_bound_euler_step_maintex}
Assume the reward function is bounded, satisfying $\sup_x |r(x)| \leq R_{\rm max}$.
Further, suppose the parameter $\theta$ is defined on a compact set $\Theta$ and velocity satisfies $u_t^\theta(y,x)\in C^2([0,T]\times\Theta)$ for all $x,y\in\mathcal{S}$.
Let $\tilde{p}_t^\theta(x)$ denote the exact distribution generated by Kolmogorov equation
\begin{align*}
    \dv{\tilde{p}_t^\theta(y)}{t} = \sum_{x \in \mathcal{S}} u_t^\theta(y,x) \tilde{p}_t^\theta(x),
\end{align*} 
and $p_t^\theta$ denote the distribution generated by Euler method \eqref{eq:one_step_prob}.
Let $\tilde{J}(\theta)$ and $J(\theta)$ denote the expected reward of $\tilde{p}_T^\theta$ and $p_T^\theta$ following \eqref{eq:RL_loss}.
Then we have:
\begin{align*}
    | J(\theta)-\tilde{J}(\theta)|=O(\Delta t),
    \|\nabla_\theta J(\theta)-\nabla_\theta\tilde{J}(\theta)\|_\infty=O(\Delta t).
\end{align*}
\end{theorem}

\begin{proof}
See \cref{subsec:proof_of_discretization_error} for a detailed proof.
\end{proof}

\cref{thm:error_bound_euler_step_maintex} shows that optimizing the RL fine-tuning objective induced by the Euler discretization gives a first-order approximation to optimizing the exact continuous-time objective.
In particular, both the expected reward and its gradient differ from their continuous-time counterparts by at most $O(\Delta t)$.
Therefore, policy-gradient updates based on the discretized sampler remain consistent with the underlying continuous-time model when the step size is sufficiently small.

\subsection{TV Error Bounds with Cross-Entropy and Generalized KL Losses}
In this section, we derive upper bounds on the terminal Total Variation distance in terms of two tractable quantities: the cross-entropy loss $\mathcal{L}_{\rm reg}^{\rm CE}(\theta;\theta^{\rm ref})$ and the generalized KL loss $\mathcal{L}_{\rm reg}^{\rm gKL}(\theta;\theta^{\rm ref})$.

\begin{theorem}[TV-Distance Error Bounds with Cross-Entropy Loss ]\label{thm:regularization_ce_maintex}
Fix a reference parameter $\theta_{{\rm ref}}$.
Assume the DFM model is parameterized through the distributions $p_{1 | t}^{\theta}(\cdot | x)$, and suppose the corresponding velocity fields $u_t^{\theta}(y, x)$ are uniformly bounded for all $x, y \in S$, $t \in [0, T]$.
Let $p^\theta$ and $p^{\theta_{{\rm ref}}}$ represent the distribution generated by $\{u_t^{\theta}\}$ and $\{u_t^{\theta_{\rm ref}}\}$ respectively.
Then it holds
\begin{align*}
{\rm TV}\bigl(p^\theta, p^{\theta_{{\rm ref}}}\bigr)
\lesssim
\sqrt{\mathcal{L}_{{\rm reg}}^{{\rm CE}}(\theta,\theta_{\rm ref})-\mathcal{L}_{{\rm reg}}^{{\rm CE}}(\theta_{\rm ref};\theta_{\rm ref})}.
\end{align*}
\end{theorem}

\begin{proof}
See \cref{subsec:tv_bound_regularization} for a detailed proof.
\end{proof}

\begin{theorem}[TV-Distance Error Bounds with Generalized KL Loss]\label{thm:regularization_general_kl_maintex}
Fix a reference parameter $\theta_{{\rm ref}}$.
Suppose the factorized velocity fields $u_t^{\theta}(y, x)$ are uniformly bounded for all $x, y \in S$, $t \in [0, T]$.
Let $p^\theta$ and $p^{\theta_{{\rm ref}}}$ represent the distribution generated by $\{u_t^{\theta}\}$ and $\{u_t^{\theta_{\rm ref}}\}$ respectively.
Then it holds
\begin{align*}
        {\rm TV}(p^{\theta},p^{\theta_{\rm ref}})\lesssim \sqrt{\mathcal{L}_{\mathrm{reg}}^{\mathrm{gKL}}(\theta;\theta_{\mathrm{ref}})}.
\end{align*}
\end{theorem}

\begin{proof}
See \cref{subsec:tv_bound_regularization} for a detailed proof.
\end{proof}

These bounds let us control the discrepancy between the fine-tuned model and the reference model.
As a result, they prevent excessive distributional drift during reward fine-tuning.
By controlling the terminal TV distance, the model stays close to the pretrained distribution while still improving reward, which helps avoid over-optimization and preserve sample naturalness.

%% file: 3exp.tex
In this section, we evaluate our approach on regulatory DNA design and compare it with established diffusion-based baselines.

\subsection{Task: Regulatory DNA Sequence Design}
\label{sec:setting}
Recent genomic models have shown that large-scale training learns transferable representations of DNA sequence structure and function \cite{zhou2025genomeocean, zhou2025dnabert, wu2025genome}.
In our setting, we focus on optimizing regulatory DNA elements so they can direct gene expression in specific cell types.
This is a key challenge in cell and gene therapy \cite{taskiran2024cell}.

\paragraph{Dataset and setting.}
We conduct experiments on the enhancer sequence design in the HepG2 cell line.
We follow the standard datasets and reward models used in prior work on computational enhancer design \cite{yang2025regulatory, wang2024fine, lal2024designing, sarkar2024designing, gosai2024machine}.
We use a publicly available large-scale enhancer dataset \cite{gosai2024machine}.
It contains activity measurements for approximately 700,000 DNA sequences of length 200 bp in human cell lines.
In this dataset, each sequence is associated with its measured expression output.
We pre-train the discrete flow-matching \cite{gat2024discrete} model on the full set of sequences.
We then partition the dataset and train two separate reward oracles, one for finetuning and the other for evaluation.
Both reward models adopt the Enformer \cite{avsec2021effective} architecture to predict enhancer activity in the HepG2 cell line.

\paragraph{Evaluations.}
We use two functional metrics and one naturalness to evaluate the generated sequences: 
\begin{itemize}
     \item \textbf{Predicted Activity based on the Evaluation Reward Oracle (Pred-Activity):} 
     We use the reward oracle trained on the evaluation split to estimate enhancer activity in the HepG2 cell line.
     Higher scores indicate more functional sequences.

    \item \textbf{Chromatin Accessibility (ATAC-Acc):}
    We use an independent binary classifier trained on HepG2 chromatin accessibility data \cite{Consortium2012} to assess whether the designed sequences are likely to lie in accessible chromatin regions.
    This is a key property of active enhancers.
    Higher scores indicate more active sequences.

    \item \textbf{3-mer Pearson Correlation (3-mer Corr - All):}
    We compute the Pearson correlation between the 3-mer frequency profile of the synthetic sequences and that of all HepG2 sequences in \cite{gosai2024machine}. 
    Higher correlation indicates that the generated sequences more closely match the overall distribution of natural enhancer sequences. 
    Compared with the ``approximated log-likelihood of sequences'' used in \cite{wang2024fine}, this metric provides a more model-independent measure of sequence naturalness, since it does not depend on the particular pre-trained model used for scoring.
\end{itemize}

\paragraph{Baselines.}
We compare our method with several baselines, including pre-trained models and direct reward-maximization approaches for controlled sequence generation in \cite{wang2024fine}.

\begin{itemize}
    \item Pre-trained Diffusion:
    This baseline uses a pre-trained discrete diffusion model by \citet{wang2024fine} to generate sequences without task-specific finetuning.

    \item Pre-trained Flow Matching:
    Following \cite{gat2024discrete,lipman2024flow}, this baseline uses a pre-trained discrete flow-matching model to generate sequences without task-specific finetuning.

    \item DRAKES \cite{wang2024fine}:
    This baseline applies reinforcement learning to optimize DNA sequences in a single pass. 
    The original method includes a KL regularization term to preserve sequence naturalness. 
    Here, we remove the KL term to isolate the effect of policy gradient optimization on functional metrics.

    \item DRAKES with KL \cite{wang2024fine}:
    This is the original DRAKES formulation with KL regularization.
    The KL term constrains the fine-tuned model to remain close to the pre-trained reference model, improving sequence naturalness.

    \item SEPO \cite{zekri2025fine}:
    This baseline applies score entropy policy optimization (SEPO) to finetune discrete diffusion models over non-differentiable rewards.
    Unlike DRAKES, SEPO does not rely on direct reward backpropagation through the full sampling trajectory and instead performs policy optimization in the diffusion setting.

    \item SEPO with GF \cite{zekri2025fine}:
    This variant augments SEPO with gradient flow (GF).
    It adds corrector-style refinement during sampling.
    It provides a stronger diffusion-based policy-optimization baseline and improves sample quality.
\end{itemize}

\subsection{Experimental Results}
\label{sec:results}
We report the results in two parts. \cref{table:functional_metrics} compares different methods without regularization loss.
\cref{table:regularization} examines the effect of adding regularization during fine-tuning.
Overall, our methods outperform the previous state-of-the-art baselines on both functional performance and sequence naturalness.
Moreover, regularization further improves sequence naturalness while preserving strong functional performance.
This yields a better trade-off between reward optimization and alignment with the natural sequence distribution.

\begin{table*}[htbp]
\centering
\caption{\textbf{Performance of DNA design methods on the HepG2 enhancer design task.} 
We compare pre-trained generative models, prior reward-driven baselines (DRAKES and SEPO), and our policy-gradient fine-tuning methods for discrete flow matching.
We do not apply regularization loss during fine-tuning in this setting.
Pred-Activity and ATAC-Acc evaluate functional performance, while 3-mer Corr-All measures sequence naturalness relative to the HepG2 data distribution.
Our methods outperform the state-of-the-art baseline SEPO on Pred-Activity and 3-mer Corr-All and achieve comparable performance on ATAC-Acc.
Bold and underlined entries highlight the best and second-best reward-driven methods.}
\label{table:functional_metrics}
\begin{tabular}{lccccc}
\toprule
\textbf{Method} & \textbf{Pred-Activity} $\uparrow$ & \textbf{ATAC-Acc} $\uparrow$ (\%) & \textbf{3-mer Corr - All} $\uparrow$ \\%& \textbf{3-mer Corr} $\uparrow$ & \textbf{JASPAR Corr} $\uparrow$ \\
\midrule
Pre-trained Diffusion & 0.17 (0.01) & 1.5 (0.3) & 0.925 (0.004) \\%& -0.054 (0.061) & -0.105 (0.079) \\
Pre-trained Flow-matching  & 0.64 (0.01) & 1.1 (0.4) & 0.884 (0.004) \\%& 0.453 (0.276) & 0.467 (0.181) \\
\midrule
DRAKES    & 6.37 (0.04) & 96.1 (0.5) & -0.379 (0.009) \\%& 0.416 (0.006) & 0.481 (0.043) \\
SEPO & 7.55 (0.01) &  \textbf{99.5} (0.2)  & -0.537 (0.002)  \\%& 0.460 (0.006) & 0.126 (0.069) \\
DoMinO-REINFORCE               & \underline{8.32} (0.01) & 99.2 (0.2) & \textbf{-0.285} (0.001) \\%& -0.080 (0.004) & 0.186 (0.036)
DoMinO-PPO          & \textbf{8.35} (0.00) & \underline{99.2} (0.2) & \underline{-0.331} (0.001) \\%& -0.215 (0.002) & 0.245 (0.001) 
\bottomrule
\end{tabular}
\end{table*}

From \cref{table:functional_metrics}, we observe that our policy-gradient fine-tuning methods outperform the prior reward-driven baselines on Pred-Activity, while remaining competitive on ATAC-Acc and achieving better sequence naturalness than SEPO.
Compared with DRAKES, DoMinO-REINFORCE and DoMinO-PPO improve Pred-Activity from 6.37 to 8.32 and 8.35, respectively, and also increase ATAC-Acc from 96.1\% to 99.2\%.
Moreover, DRAKES yields a 3-mer Corr-All of -0.379, while our methods obtain -0.285 and -0.331.
This indicates better preservation of sequence naturalness.
Compared with SEPO, our methods further improve Pred-Activity from 7.55 to 8.32 and 8.35, while achieving better sequence naturalness (improving 3-mer Corr-All from -0.537 to -0.285 and -0.331).
Although SEPO attains a slightly higher ATAC-Acc (99.5\%) than our unregularized methods (99.2\%), the overall results show that our methods achieve a more favorable trade-off between functional optimization and sequence naturalness.

The results in \cref{table:functional_metrics} also suggest that discrete flow matching provides a stronger backbone for reward-based optimization than discrete diffusion.
Before fine-tuning, the pre-trained flow-matching model already achieves higher Pred-Activity than the pre-trained diffusion model (0.64 vs.\ 0.17), although both models perform poorly on ATAC-Acc.
After reward-driven fine-tuning, the flow-matching-based methods improve both functional metrics and reach Pred-Activity above 8.3 and ATAC-Acc of 99.2\%.
These results suggest that discrete flow matching provides a better foundation for controllable regulatory DNA design than diffusion-based alternatives.

\begin{table*}[htbp]
\centering
\caption{\textbf{Effect of regularization on DoMinO for the HepG2 enhancer design task.}
We compare our policy-gradient fine-tuning methods with and without regularization loss.
Pred-Activity and ATAC-Acc measure functional performance, while 3-mer Corr-All measures sequence naturalness relative to the HepG2 data distribution.
Regularization improves sequence naturalness and helps maintain strong functional performance.
This yields a better balance between reward optimization and naturalness.
Bold and underlined entries highlight the best and second-best regularized reward-driven methods.
}
\label{table:regularization}
\begin{tabular}{lccccc}
\toprule
\textbf{Method} & \textbf{Pred-Activity} $\uparrow$ & \textbf{ATAC-Acc} $\uparrow$ (\%) & \textbf{3-mer Corr - All} $\uparrow$ \\%& \textbf{3-mer Corr} $\uparrow$ & \textbf{JASPAR Corr} $\uparrow$ \\
\midrule
DRAKES    & 6.37 (0.04) & 96.1 (0.5) & -0.379 (0.009) \\%& 0.416 (0.006) & 0.481 (0.043) \\
DRAKES with KL    & 5.61 (0.07) & 92.5 (0.6) & -0.302 (0.011) \\%& 0.416 (0.006) & 0.481 (0.043) \\
\midrule
SEPO & 7.55 (0.01) &  99.5 (0.2)  & -0.537 (0.002)  \\%& 0.460 (0.006) & 0.126 (0.069) \\
SEPO with GF & 7.64 (0.01) & \textbf{99.9} (0.09) & -0.496 (0.001) \\%& 0.215 (0.056) & 0.126 (0.069) \\
\midrule
DoMinO-REINFORCE               & 8.32 (0.01) & 99.2 (0.2) & -0.285 (0.001) \\%& -0.080 (0.004) & 0.186 (0.036)
DoMinO-REINFORCE with CE                & \underline{8.22} (0.01) & 94.1 (0.8) & -0.347 (0.004) \\%& -0.080 (0.004) & 0.186 (0.036)
DoMinO-REINFORCE with GKL             & \textbf{8.24} (0.03) & 90.2 (0.7) & \textbf{0.013} (0.003) \\%& -0.080 (0.004) & 0.186 (0.036)
\midrule
DoMinO-PPO          & 8.35 (0.00) & 99.2 (0.2) & -0.331 (0.001) \\%& -0.215 
DoMinO-PPO with CE              & 7.98 (0.01) & \underline{97.5} (0.3) & \underline{-0.152} (0.002) \\%& -0.080 (0.004) & 0.186 (0.036)
DoMinO-PPO with GKL           & 7.78 (0.01) & 95.8 (0.4) & -0.167 (0.001) \\%& -0.080 (0.004) & 0.186 (0.036)
\bottomrule
\end{tabular}
\end{table*}

\cref{table:regularization} further clarifies the role of regularization.
For the DRAKES and SEPO baselines, regularization improves sequence naturalness but only modestly: DRAKES with KL improves 3-mer Corr-All from -0.379 to -0.302, and SEPO with GF improves it from -0.537 to -0.496.
In our method, regularization leads to a much larger gain in sequence naturalness.
For example, DoMinO-REINFORCE with GKL improves 3-mer Corr-All from -0.285 to 0.013.
This makes it the only method with positive 3-mer Corr-All, while retaining strong Pred-Activity of 8.24.
Similarly, DoMinO-PPO with CE and GKL improves 3-mer Corr-All from -0.331 to -0.152 and -0.167, respectively.
These results show that regularization is effective in our framework for improving alignment with the HepG2 sequence distribution while preserving strong functional performance.

Overall, the results support policy-gradient fine-tuning of discrete flow matching as an effective approach for enhancer design.
Without regularization, our methods already outperform prior reward-driven baselines on Pred-Activity and achieve better sequence naturalness than SEPO.
With regularization, they further improve the trade-off between functional performance and naturalness, with DoMinO-REINFORCE with GKL providing the strongest overall balance among the compared methods.

%% file: 4conclusion.tex
In this work, we introduce \underline{D}iscrete fl\underline{o}w \underline{M}atch\underline{in}g policy \underline{O}ptimization  (DoMinO), a unified reinforcement learning fine-tuning framework for discrete flow matching models.
Building on prior work for continuous-space generative models \cite{black2023training}, we extend policy gradient fine-tuning to the discrete flow matching setting.
DoMinO supports standard policy gradient methods, handles both conditional and unconditional generation, works with non-differentiable rewards, and avoids additional approximation.
Under this framework, we develop two concrete algorithms, DoMinO-REINFORCE and DoMinO-PPO.
To prevent over-optimization and preserve the sample naturalness, we further introduce two regularizers, based on cross-entropy and generalized KL, to control the Total Variation distance between the pretrained model and fine-tuned model.

Theoretically, we justify DoMinO from two directions.
First, we analyze the discretization error induced by the Euler sampler and show that both the reward fine-tuning objective and its gradient incur only $O(\Delta t)$ error relative to their exact continuous-time counterparts.
Second, we derive upper bounds on the terminal Total Variation distance in terms of the cross-entropy and generalized KL regularizers.
Experimentally, our experiments on regulatory DNA sequence design shows the effectiveness of DoMinO.
In particular, it achieves state-of-the-art predicted enhancer activity and chromatin accessibility, while maintaining the natural sequence distribution.

%% file: impact.tex
This work advances the methodological and theoretical understanding of RL fine-tuning for discrete flow matching models and presents no foreseeable negative social impacts.

%% file: x_acknowledgments.tex
JH is partially supported by Northwestern University’s Walter P. Murphy Fellowship and Terminal Year Fellowship (Paul K. Richter Memorial Award).
Han Liu is partially supported by NIH R01LM1372201, NSF
AST-2421845, Simons Foundation
MPS-AI-00010513, AbbVie , Dolby and Chan Zuckerberg Biohub Chicago Spoke Award.
This research was supported in part through the computational resources and staff contributions provided for the Quest high performance computing facility at Northwestern University which is jointly supported by the Office of the Provost, the Office for Research, and Northwestern University Information Technology.
The content is solely the responsibility of the authors and does not necessarily represent the official
views of the funding agencies.

%% file: appendix.tex
\section{Proofs of Main Text}
This section details the proofs of the theoretical results presented in the main text.

\subsection{\texorpdfstring{Proof of \cref{thm:error_bound_euler_step_maintex}}{}}\label{subsec:proof_of_discretization_error}
We first restate the discrete Grönwall's inequality, which serves as the foundation for bounding the accumulated discretization errors.
\begin{lemma}[Discrete Grönwall's Lemma]\label{lem:discrete_gronwall_inequality}
Let $\{y_n\}$ be a sequence of non-negative real numbers in $\R$.
Assume that $a,b$ are non-negative constants.
Suppose for all $n \geq 1$, it holds:
\begin{align*}
    y_n\leq(1+a)y_{n-1}+b.
\end{align*}
Then we have
\begin{align*}
    y_n \leq (1+a)^ny_0+b\sum_{j=0}^{n-1}(1+a)^j.
\end{align*}
\end{lemma}

\begin{proof}
This follows directly by mathematical induction on $n$.
\end{proof}

We provide proof of \cref{thm:error_bound_euler_step_maintex} on basis of \cref{lem:discrete_gronwall_inequality}.

\begin{theorem}[\cref{thm:error_bound_euler_step_maintex} Restated]\label{thm:error_bound_euler_step}
Assume the reward function is bounded, satisfying $\sup_x |r(x)| \leq R_{\rm max}$.
Further, suppose the parameter $\theta$ is defined on a compact set $\Theta$ and velocity satisfies $u_t^\theta(y,x)\in C^2([0,T]\times\Theta)$ for all $x,y\in\mathcal{S}$.
Let $\tilde{p}_t^\theta(x)$ denote the exact distribution generated by Kolmogorov equation
\begin{align*}
    \dv{\tilde{p}_t^\theta(y)}{t} = \sum_{x \in \mathcal{S}} u_t^\theta(y,x) \tilde{p}_t^\theta(x),
\end{align*} 
and $p_t^\theta$ denote the distribution generated by Euler method \eqref{eq:one_step_prob}.
Let $\tilde{J}(\theta)$ and $J(\theta)$ denote the expected reward of $\tilde{p}_T^\theta$ and $p_T^\theta$  following \eqref{eq:RL_loss}.
Then we have
\begin{align*}
    | J(\theta)-\tilde{J}(\theta)|=O(\Delta t),
    \|\nabla_\theta J(\theta)-\nabla_\theta\tilde{J}(\theta)\|_\infty=O(\Delta t).
\end{align*}
\end{theorem}

\begin{proof}
Without loss of generality, we enumerate the state space as $\mathcal{S} = \{x_1, \dots, x_{|\mathcal{S}|}\}$.
We define the velocity matrix $U_t^\theta\in\R^{|\mathcal{S}|\times|\mathcal{S}|}$ as $(U_t^\theta)_{j,k}=u_t^\theta(x_j,x_k)$.
Then $U_t^\theta\in C^2$.

Recall the evolution of the exact probability distribution $\tilde{p}_t^\theta$ in continuous time follows the Kolmogorov forward equation
\begin{align*}
\frac{{\rm d}}{{\rm d}t} \tilde{p}_t^\theta = U_t^\theta \tilde{p}_t^\theta.
\end{align*}
Correspondingly, the evolution of discrete probability distribution $p_k^\theta$ using the Euler method with step size $\Delta t$ follows
\begin{align*}
p_{k+1}^\theta = (I + \Delta t U_{t_k}^\theta) p_k^\theta.
\end{align*}
By representing the reward as a vector $r\in\R^{|\mathcal{S}|}$, we express the gradient of target function as
\begin{align*}
    \nabla_\theta J(\theta)=\sum_{x\in\mathcal{S}}r(x)\nabla_\theta p_T^\theta(x)=r^\top\nabla_\theta p_T^\theta(x),\\
    \nabla_\theta \tilde{J}(\theta)=\sum_{x\in\mathcal{S}}r(x)\nabla_\theta \tilde{p}_T^\theta(x)=r^\top\nabla_\theta \tilde{p}_T^\theta(x).
\end{align*}
We bound the gradient error as
\begin{align*}
    \|\nabla_\theta J(\theta)-\nabla_\theta \tilde{J}(\theta)\|_\infty\leq R_{\rm max}\cdot \|\nabla_\theta p_T^\theta(x)-\nabla_\theta \tilde{p}_T^\theta(x)\|_1.
\end{align*}
Similarly, for the value objective, we have
\begin{align*}
    | J(\theta)-\tilde{J}(\theta)|\leq R_{\rm max}\|p_{T}^\theta-\tilde{p}_{T}^\theta\|_1.
\end{align*}
Since $u_t^\theta(y,x)\in C^2([0,T]\times\Theta)$ for all $x,y\in\mathcal{S}$ and $[0,T]\times\Theta$ is compact, we suppose $\|U_t^\theta\|_1\leq M_U, \|\nabla_\theta U_t^\theta\|_1\leq L_\theta$. 
For simplicity of notation, we set $t_k=k\Delta t$ and $K=\frac{T}{\Delta t}$. 
We use $\epsilon_k:=p_{t_k}^\theta-\tilde{p}_{t_k}^\theta, ~E_k:=\nabla_\theta p^\theta_{t_k}-\nabla_\theta \tilde{p}^\theta_{t_k}$ to denote the accumulated discretization error at time $t_k$.
Note that at $t=0$, we have exact initialization, meaning $\epsilon_0 = 0$ and $E_0 = 0$.

For the continuous gradient, differentiating the Kolmogorov equation yields
\begin{align*}
    \frac{{\rm d}}{{\rm d}t} (\nabla_\theta \tilde{p}_t^\theta) = (\nabla_\theta U_t^\theta) \tilde{p}_t^\theta + U_t^\theta (\nabla_\theta \tilde{p}_t^\theta).
\end{align*}
A first-order Taylor expansion gives the continuous gradient evolution over one time step $\Delta t$:
\begin{align}\label{eq:continuous_gradient_p}
    \nabla_\theta \tilde{p}_{t_{k+1}}^\theta = \nabla_\theta \tilde{p}_{t_k}^\theta + \Delta t ((\nabla_\theta U_{t_k}^\theta) \tilde{p}_{t_k}^\theta + U_{t_k}^\theta (\nabla_\theta \tilde{p}_{t_k}^\theta) ) + R_p,
\end{align}
where the remainder term $R_p = \mathcal{O}(\Delta t^2)$. 
Since $U_t^\theta \in C^2$, there exists a constant $C_p > 0$ such that $\|R_p\|_1 \le C_p\Delta t^2$.
Recall that under \eqref{eq:one_step_prob} we have
\begin{align}\label{eq:equation_of_discrete_p}
    p_{t_{k+1}}^\theta=(I+\Delta tU_{t_k}^\theta)p_{t_{k}}^\theta.
\end{align}
In parallel, differentiating the discrete Euler transition step yields
\begin{align}\label{eq:discrete_gradient_p}
    \nabla_\theta p_{t_{k+1}}^\theta = \nabla_\theta p_{t_k}^\theta + \Delta t ( (\nabla_\theta U_{t_k}^\theta) p_{t_k}^\theta + U_{t_k}^\theta (\nabla_\theta p_{t_k}^\theta) ).
\end{align}
Comparing \eqref{eq:continuous_gradient_p} and \eqref{eq:discrete_gradient_p}, we obtain the recursive error formulation for the gradient:
\begin{align}\label{eq:equation_of_discrete_error}
    E_{k+1}=(I+\Delta tU_{t_k}^\theta)E_k+\Delta t(\nabla_\theta U_{t_k}^\theta)(p_{t_k}^\theta-\tilde{p}_{t_k}^\theta)-R_p.
\end{align}
With Taylor expansion, we have
\begin{align}\label{eq:equation_of_continuous_p}
    \tilde{p}_{t_{k+1}}^\theta=(I+\Delta tU_{t_k}^\theta)\tilde{p}_{t_{k}}^\theta+R_k,
\end{align}
where the remainder term $R_k = \mathcal{O}(\Delta t^2)$.
Further, since $\frac{\dd}{\dd t}U_t^\theta$ and $\nabla_\theta U_t^\theta$ is bounded, there exists a constant $C_1>0$ such that $\|R_k\|_1\leq C_1\Delta t^2$.
Then by \eqref{eq:equation_of_discrete_p} and \eqref{eq:equation_of_continuous_p}, we have
\begin{align*}
    \epsilon_{k+1}=(I+\Delta tU_{t_k}^\theta)\epsilon_k -R_k.
\end{align*}
Taking the $L_1$ norm on both sides generates
\begin{align*}
    \|\epsilon_{k+1}\|_1\leq\|I+\Delta tU_{t_k}^\theta\|_1\|\epsilon_k\|_1 +C_1\Delta t^2.
\end{align*}
Applying \cref{lem:discrete_gronwall_inequality}, we establish that
\begin{align*}
    \|\epsilon_k\|_1\leq & ~ C_1\Delta t^2\sum_{j=0}^{k-1}(1+M_U\Delta t)^j \\
    \leq & ~
    C_1\Delta t^2\frac{(1+M_U\Delta t)^\frac{T}{\Delta t}}{M_U\Delta t}\\
    \leq & ~
    \frac{C_1\exp(M_UT)}{M_U}\Delta t.
\end{align*}
Therefore, there exists $C_2>0$ such that $\|\epsilon_k\|_1\leq C_2\Delta t$.
Consequently, we have
\begin{align*}
    | J(\theta)-\tilde{J}(\theta)|\leq R_{\rm max}\|\epsilon_K\|_1=O(\Delta t).
\end{align*}
Returning to the gradient error $E_k$.
Taking the $L_1$ norm on both sides of \eqref{eq:equation_of_discrete_error}, we obtain
\begin{align*}
    \|E_{k+1}\|_1\leq\|I+\Delta tU_{t_k}^\theta\|_1\|E_k\|_1+ (C_p+C_2L_\theta)\Delta t^2.
\end{align*}
Similar to bound on $\|\epsilon_k\|$, by applying \cref{lem:discrete_gronwall_inequality} we have
\begin{align*}
    \|E_k\|_1\leq & ~ (C_p+C_2L_\theta)\Delta t^2\sum_{j=0}^{k-1}(1+M_U\Delta t)^j \\
    \leq & ~
    (C_p+C_2L_\theta)\Delta t^2\frac{(1+M_U\Delta t)^\frac{T}{\Delta t}}{M_U\Delta t}\\
    \leq & ~
    \frac{(C_p+C_2L_\theta)\exp(M_UT)}{M_U}\Delta t.
\end{align*}
This implies
\begin{align*}
    \|\nabla_\theta J(\theta)-\nabla_\theta \tilde{J}(\theta)\|_\infty\leq R_{\rm max}\|E_K\|_1=O(\Delta t).
\end{align*}
This completes the proof.
\end{proof}

\subsection{Proof of TV-Distance guarantees for Distribution Regularization }\label{subsec:tv_bound_regularization}

For simplicity of notation, we use the vector value function $u_\theta(x,t)$ to represent the velocity field of $x\in\mathcal{S}$ on dimension at time $t\in[0,T]$ generated by parameter $\theta$, satisfying  $u_\theta(x,t)=u_t^{\theta}(\cdot,x)$.
We start with introducing a lemma bounding the distribution TV distance with integration of velocity $\ell_2$ difference.

\begin{lemma}[Theorem C.1 of \cite{su2025theoretical}]\label{lem:bound_distribution_with_l2}
Given a fixed $\theta_{\rm ref}$, for other $\theta$ we define the factorized risk as the mean squared error of its velocity:
\begin{align*}
    \mathcal{R}(\theta):=\int_0^T\E_{X_t\sim p_t^{\theta_{\rm ref}}}\|u^{\theta}(X_t,t)-u^{\theta_{\rm ref}}(X_t,t)\|_2^2\dd t.
\end{align*}
Suppose the factorized velocity $u_t^{\theta}(y,x)$ is bounded for $x,y\in\mathcal{S},t\in[0,T]$.
Then the total variation distance between the distributions $p_T^\theta$ and $p_T^{\theta_{\rm ref}}$ follows
\begin{align*}
    {\rm TV}(p_T^\theta,p_T^{\theta_{\rm ref}})\lesssim\sqrt{\mathcal{R}(\theta)}.
\end{align*}
\end{lemma}

\begin{proof}
See proof of \cite[Theorem C.1]{su2025theoretical}.
\end{proof}

Next, we present the proof of \cref{thm:regularization_ce_maintex}.

\begin{theorem}[\cref{thm:regularization_ce_maintex} Restated]\label{thm:regularization_ce}
Fix a reference parameter $\theta_{{\rm ref}}$.
Assume the DFM model is parameterized through the distributions $p_{1 | t}^{\theta}(\cdot | x)$, and suppose the corresponding velocity fields $u_t^{\theta}(y, x)$ are uniformly bounded for all $x, y \in S$, $t \in [0, T]$.
Let $p^\theta$ and $p^{\theta_{{\rm ref}}}$ represent the distribution generated by $\{u_t^{\theta}\}$ and $\{u_t^{\theta_{\rm ref}}\}$ respectively.
Then it holds
\begin{align*}
{\rm TV}\bigl(p^\theta, p^{\theta_{{\rm ref}}}\bigr)
\lesssim
\sqrt{\mathcal{L}_{{\rm reg}}^{{\rm CE}}(\theta,\theta_{\rm ref})-\mathcal{L}_{{\rm reg}}^{{\rm CE}}(\theta_{\rm ref};\theta_{\rm ref})}.
\end{align*}
\end{theorem}

\begin{proof}
We define that $H(p)={\rm CE}(p,p)$.
By definition of KL divergence, it holds
\begin{align*}
    {\rm CE}(p,q)=H(p)+{\rm KL}(p\parallel q).
\end{align*}
By \eqref{eq:dfm_ce_reg}, we obtain
\begin{align*}
    \mathcal{L}_{{\rm reg}}^{{\rm CE}}(\theta,\theta_{\rm ref})=\mathcal{L}_{{\rm reg}}^{{\rm CE}}(\theta_{\rm ref},\theta_{\rm ref})+\E_{t,X_t\sim p_t^{\theta_{\rm ref}}}[{\rm KL}(p_{1|t}^{\theta_{\rm ref}}(\cdot | X_t)\parallel p_{1|t}^{\theta}(\cdot | X_t))].
\end{align*}
That's to say, for parameter $\theta$, the regularization condition $\mathcal{L}_{{\rm reg}}^{{\rm CE}}(\theta,\theta_{\rm ref})-\mathcal{L}_{{\rm reg}}^{{\rm CE}}(\theta_{\rm ref},\theta_{\rm ref})\leq\epsilon$ is equivalent to $\E_{t,X_t\sim p_t^{\theta_{\rm ref}}}[{\rm KL}(p_{1|t}^{\theta_{\rm ref}}(\cdot | X_t)\parallel p_{1|t}^{\theta}(\cdot | X_t))]\leq\epsilon$.

Recall that under mixture path setting, the velocity admits the form
\begin{align}
    u_t(y,x)= & ~
    \sum_{x_1}\frac{\dot{\kappa_t}}{1-\kappa_t}[\delta(y,x_1)-\delta(y,x)]p_{1|t}(x_1|x) \notag\\
    = & ~
    \begin{cases}\label{eq:form_of_mixture_path_velocity}
        \frac{\dot{\kappa_t}}{1-\kappa_t}p_{1|t}(y|x),\hfill &{\rm for~} y\neq x,\\
        -\frac{\dot{\kappa_t}}{1-\kappa_t}\sum_{x_1\neq x}p_{1|t}(x_1|x),\hfill &{\rm for~} y= x.
    \end{cases}
\end{align}
The factorized risk then takes the form
\begin{align*}
    \mathcal{R}(\theta)= & ~
    \int_0^T\E_{X_t\sim p_t^{\theta_{\rm ref}}}\frac{\dot{\kappa_t}^2}{(1-\kappa_t)^2} (\sum_{y\neq x}(p_{1|t}^{\theta}(y|x)-p_{1|t}^{\theta_{\rm ref}}(y|x))^2 \\
    ~ & \qquad\qquad\qquad\qquad\qquad\qquad\qquad 
    +(\sum_{x_1\neq x}p_{1|t}^{\theta}(x_1|x)-\sum_{x_1\neq x}p_{1|t}^{\theta_{\rm ref}}(x_1|x))^2) \dd t \\
    = & ~
    \int_0^T\E_{X_t\sim p_t^{\theta_{\rm ref}}}\frac{\dot{\kappa_t}^2}{(1-\kappa_t)^2} (\sum_{y\neq x}(p_{1|t}^{\theta}(y|x)-p_{1|t}^{\theta_{\rm ref}}(y|x))^2\\
    ~ & \qquad\qquad\qquad\qquad\qquad\qquad\qquad +((1-\sum_{x_1\neq x}p_{1|t}^{\theta}(x_1|x))-(1-\sum_{x_1\neq x}p_{1|t}^{\theta_{\rm ref}}(x_1|x)))^2) \dd t\\
    = & ~
    \int_0^T\E_{X_t\sim p_t^{\theta_{\rm ref}}} \frac{\dot{\kappa_t}^2}{(1-\kappa_t)^2}\|p_{1|t}^{\theta}(\cdot|X_t)-p_{1|t}^{\theta_{\rm ref}}(\cdot|X_t)\|_2^2 \dd t.
\end{align*}
For valid selection of $\kappa_t$, $\frac{\dot{\kappa_t}}{1-\kappa_t}$ is bounded on $[0,T]$.
We assume $\frac{\dot{\kappa_t}}{(1-\kappa_t)}\leq M_0$ for some fixed $M_0>0$.
Then we have
\begin{align*}
    \mathcal{R}(\theta)\leq & ~
    M_0^2\int_0^T\E_{X_t\sim p_t^{\theta_{\rm ref}}}\|p_{1|t}^{\theta}(\cdot|X_t)-p_{1|t}^{\theta_{\rm ref}}(\cdot|X_t)\|_2^2 \dd t\\
    \leq & ~
    M_0^2\int_0^T\E_{X_t\sim p_t^{\theta_{\rm ref}}}\|p_{1|t}^{\theta}(\cdot|X_t)-p_{1|t}^{\theta_{\rm ref}}(\cdot|X_t)\|_1^2 \dd t\annot{By $\|v\|_2\leq\|v\|_1$}\\
    \leq & ~
    2M_0^2\int_0^T\E_{X_t\sim p_t^{\theta_{\rm ref}}}{\rm KL}(p_{1|t}^{\theta}(\cdot|X_t)\parallel p_{1|t}^{\theta_{\rm ref}}(\cdot|X_t)) \dd t. \annot{By Pinsker Inequality}
\end{align*}
Finally, by \cref{lem:bound_distribution_with_l2} we have
\begin{align*}
    {\rm TV}(p^{\theta},p^{\theta_{\rm ref}})
    \lesssim & ~
    \sqrt{\mathcal{R}(\theta)} \\
    \lesssim & ~
    \sqrt{\int_0^T\E_{X_t\sim p_t^{\theta_{\rm ref}}}{\rm KL}(p_{1|t}^{\theta}(\cdot|X_t)\parallel p_{1|t}^{\theta_{\rm ref}}(\cdot|X_t)) \dd t} \\
    \lesssim & ~ \sqrt{\mathcal{L}_{{\rm reg}}^{{\rm CE}}(\theta,\theta_{\rm ref})-\mathcal{L}_{{\rm reg}}^{{\rm CE}}(\theta_{\rm ref};\theta_{\rm ref})}.
\end{align*}
\end{proof}

Finally, we present proof of \cref{thm:regularization_general_kl_maintex}.

\begin{theorem}[\cref{thm:regularization_general_kl_maintex} Restated]\label{thm:regularization_general_kl}
Fix a reference parameter $\theta_{{\rm ref}}$.
Suppose the factorized velocity fields $u_t^{\theta}(y, x)$ are uniformly bounded for all $x, y \in S$, $t \in [0, T]$.
Let $p^\theta$ and $p^{\theta_{{\rm ref}}}$ represent the distribution generated by $\{u_t^{\theta}\}$ and $\{u_t^{\theta_{\rm ref}}\}$ respectively.
Then it holds
\begin{align*}
        {\rm TV}(p^{\theta},p^{\theta_{\rm ref}})\lesssim \sqrt{\mathcal{L}_{\mathrm{reg}}^{\mathrm{gKL}}(\theta;\theta_{\mathrm{ref}})}.
\end{align*}
\end{theorem}

\begin{proof}
Notice that the following inequality holds for $u_j,v_j>0$
\begin{align*}
    u_j\log\frac{u_j}{v_j}-u_j+v_j\geq\frac{(u_j-v_j)^2}{2(u_j+v_j)}.
\end{align*}
Summing over $j$, we obtain
\begin{align*}
    D_{\rm gKL}(u,v)\geq\sum_j\frac{(u_j-v_j)^2}{2(u_j+v_j)}.
\end{align*}
By \eqref{eq:form_of_mixture_path_velocity}, it holds
\begin{align*}
    |u_t(y,x)|\leq\frac{\dot{\kappa_t}}{1-\kappa_t}\leq M_0.
\end{align*}
Therefore, we obtain
\begin{align*}
     D_{\rm gKL}(u_t^{\theta_{\mathrm{ref}}}(\cdot, X_t),u_t^{\theta}(\cdot, X_t)) \geq
     \frac{\|u_t^{\theta_{\mathrm{ref}}}(\cdot, X_t)-u_t^{\theta}(\cdot, X_t)\|_2^2}{16M_0}.
\end{align*}
Recall that the velocity risk satisfies
\begin{align*}
    \mathcal{R}(\theta) = \int_0^T\E_{X_t\sim p_t^{\theta_{\rm ref}}}\|u^{\theta}(X_t,t)-u^{\theta_{\rm ref}}(X_t,t)\|_2^2\dd t.
\end{align*}
Therefore
\begin{align*}
    \mathcal{R}(\theta)\leq4M_0
    \int_0^T\E_{X_t\sim p_t^{\theta_{\rm ref}}} D_{\rm gKL}(u_t^{\theta_{\mathrm{ref}}}(\cdot, X_t),u_t^{\theta}(\cdot, X_t)) \dd t.
\end{align*}
Finally, by \cref{lem:bound_distribution_with_l2} we have
\begin{align*}
    {\rm TV}(p^{\theta},p^{\theta_{\rm ref}})
    \lesssim & ~
    \sqrt{\mathcal{R}(\theta)}\\
    \lesssim & ~
    \sqrt{\mathcal{L}_{\mathrm{reg}}^{\mathrm{gKL}}(\theta;\theta_{\mathrm{ref}})}.
\end{align*}
This completes the proof.
\end{proof}